\documentclass[lettersize,journal]{IEEEtran}
\usepackage{amsmath,amsfonts}
\usepackage{algorithmic}
\usepackage{algorithm}
\usepackage{array}
\usepackage[caption=false,font=normalsize,labelfont=sf,textfont=sf]{subfig}
\usepackage{textcomp}
\usepackage{stfloats}
\usepackage{url}
\usepackage{verbatim}
\usepackage{graphicx}
\usepackage{cite}
\usepackage{color}
\newtheorem{definition}{Definition}
\newtheorem{proof}{Proof}
\newtheorem{theorem}{Theorem}
\newtheorem{lemma}{Lemma}%[theorem]

\newtheorem{remark}{remark}
\usepackage{bbm}
\usepackage{hyperref}

\begin{document}

\title{FedIL: Federated Incremental Learning from Decentralized Unlabeled Data with Convergence Analysis}

\author{\IEEEauthorblockN{Nan Yang, Dong Yuan, Charles Z Liu,  Yongkun Deng and Wei Bao \\}
\IEEEauthorblockA{\textit{Faculty of Engineering, The University of Sydney} \\
\{n.yang, dong.yuan, zhenzhong.liu\}@sydney.edu.au}, yden9681@uni.sydney.edu.au, wei.bao@sydney.edu.au}
        % <-this % stops a space
%\thanks{This paper was produced by the IEEE Publication Technology Group. They are in Piscataway, NJ.}% <-this % stops a space
%\thanks{Manuscript received April 19, 2021; revised August 16, 2021.}}



\maketitle

\begin{abstract}
%Federated learning (FL) has emerged as an effective technique to co-train machine learning models without actually sharing data and leaking privacy. However, m
Most existing federated learning methods assume that clients have fully labeled data to train on, while in reality, it is hard for the clients to get task-specific labels due to users' privacy concerns, high labeling costs, or lack of expertise. This work considers the server with a small labeled dataset and intends to use unlabeled data in multiple clients for semi-supervised learning. We propose a new framework with a generalized model, Federated Incremental Learning (FedIL), to address the problem of how to utilize labeled data in the server and unlabeled data in clients separately in the scenario of Federated Learning (FL). FedIL uses the Iterative Similarity Fusion to enforce the server-client consistency on the predictions of unlabeled data and uses incremental confidence to establish a credible pseudo-label set in each client. We show that FedIL will accelerate model convergence by Cosine Similarity with normalization, proved by Banach Fixed Point Theorem. The code is available at \href{https://anonymous.4open.science/r/fedil}{https://anonymous.4open.science/r/fedil}.

\end{abstract}


\section{Introduction}

Federated Learning (FL) is a decentralized technique where multiple clients collaborate to train a global model through coordinated communication \cite{mcmahan2017communication,zhao2018federated,chen2019communication}. This method has been successfully applied in a range of applications and offers solutions to data privacy, security, and access issues\cite{hard2018federated,yang2019ffd,brisimi2018federated}. However, previous FL research \cite{han2020robust} often assumes that clients have fully annotated data with ground-truth labels, which can be an unrealistic assumption as labeling data is a time-consuming, expensive process that often requires the participation of domain experts. A more practical scenario is to share a limited amount of labeled data on the server while assisting clients with unlabeled data in model training \cite{jeong2021federated}, as proposed in recent studies.

The samples and labels on the server side may be accurate, however, it must be acknowledged that the data on the server is limited and does not provide a complete picture. This results in the reflection of only local features, and not global features. While other clients may have more extensive data coverage, the lack of reliable label information causes the features to be unreliable and uncertain. The two main challenges in model training in an FSSL scenario are depicted in Figure \ref{challenges}. Firstly, overfitting can occur when too much reliance is placed on labeled data from the server, leading to poor model generalization. Secondly, the absence of ground-truth labels in the client data may result in the annotating of incorrect pseudo-labels, causing model mislearning and preventing convergence of model training.

Motivated by this practical scenario, a naive solution is to simply perform SSL methods using any off-the-shelf methods (e.g. FixMatch \cite{DBLP:conf/nips/SohnBCZZRCKL20}, UDA \cite{xie2020unsupervised}), while using federated learning strategies to aggregate the learned weights.
The recent method FedMatch \cite{jeong2021federated} uses existing SSL methods based on pseudo-labeling and enhances the consistency between predictions made across multiple models by deploying a second labeled dataset for validation on the server, but additionally increases the need for labeled data. 
Other approaches FedU \cite{DBLP:conf/iccv/ZhuangG0ZY21}, FedEMA \cite{zhuang2021divergence}, and Orchestra \cite{DBLP:conf/icml/LubanaTKDM22}, use self-supervised strategies to correct the training results of aggregated client models by labeled data on the server. These methods may result in a learning strategy that heavily relies on the feature information of labeled data, leading to the risk of overfitting.
Therefore, a major challenge in FSSL is finding a way to correct the training bias caused by uncertain samples and labels on clients while also learning correct feature information from clients that the server does not have. Additionally, current research in this field is based on empirical model design, with no analysis of model convergence for the FSSL scenario.

\begin{figure*}[t]
  \centering
  \includegraphics[width=0.95\textwidth]{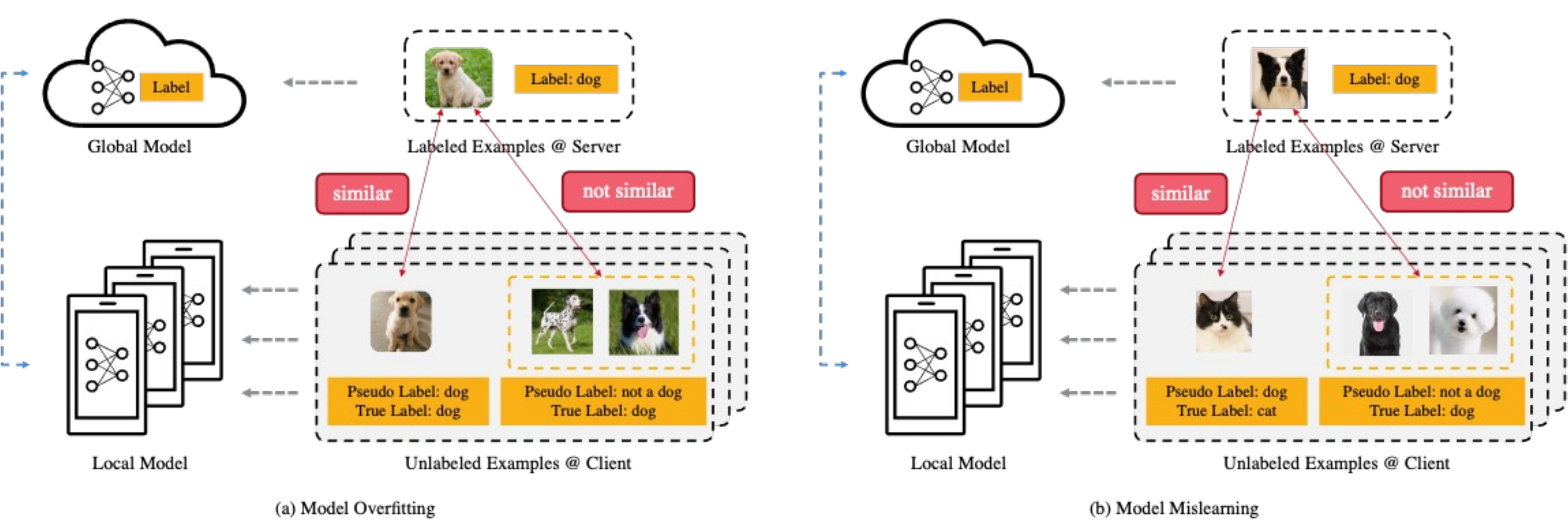}
  \caption{\textbf{Illustrations of Two Challenges in Federated Semi-Supervised Learning} (a) Model overfitting is caused by over-reliance on the labeled data on the server, leading to poor recognition of one class that is not similar to those in the training data. (b) Model mislearning is caused by incorrect pseudo-labels generated in clients, leading to incorrect classification in that the data belongs to one class but is labeled as another.}
  \label{challenges}
\end{figure*}

In this paper, we propose Federated Incremental Learning (FedIL), which is a novel and general framework, for addressing the problem of FSSL, aiming at the demanding challenge of how to utilize separated labeled data in the server and unlabeled data in clients. 
Inspired by the mainstream semi-supervised learning methods based on pseudo-labeling, we alternate the training of labeled data in the server and unlabeled data in selected clients and employ a siamese network for contrastive learning to ensure acquiring high-quality pseudo-labels during training. To prevent the model from forgetting what it has learned from labeled data in the server when learning unlabeled data in clients, our method uses KL loss during client training to enforce the consistency between the predictions made by clients and the server. 
In clients, FedIL selects high-confidence pseudo-labels obtained through the mechanism of Incremental Credibility Learning to establish an independent and highly credible pseudo-label set in each client, allowing each client to achieve authentic semi-supervised learning to reduce the training bias during the client training.
In the server, we screen the uploaded client weights by Cosine Similarity with normalization to accelerate the convergence of model training.

The design of FedIL is based on the specific needs of the FSSL scenario, especially in protecting the data privacy of clients. The convergence condition of the traditional FL model is the decrease of the Loss, while the Loss in FSSL is computed by pseudo-labels rather than ground-truth labels, which is unreliable, and thus cannot guarantee the model convergence. Therefore, we first propose a Theorem base on the Banach fixed point to use weight difference as a criterion rather than Loss to determine model convergence in FSSL, which reveals that the model enters a progressive convergence stage if the increment difference of weight between the server and the client decreases.

The main contributions of this work are summarized as follows:

\begin{itemize}
\item 
We first introduce \textbf{Incremental Credibility Learning} to select pseudo labels with higher confidence for joining a pseudo-label set to establish a real and reliable semi-supervised training in each client. 
We also propose \textbf{Global Incremental Learning} to use \textbf{Cosine Similarity} with normalization to select client weights that are close to the server weight, which accelerates model learning.

\item
%We further propose a theorem base on the \textbf{Banach fixed point} to reveal that the system enters a progressive convergence stage when the weight increment difference between server and client decreases. Therefore,  \textbf{Cosine Similarity with normalization} was used to select client weights that are close to the server weight, which accelerates model learning.
We further propose a Theorem base on the \textbf{Banach fixed point}, a new definition of convergence for FSSL, that is based on weight difference rather than the loss in supervised FL. This Theorem reveals that the system enters a progressive convergence stage when the weight increment difference between the server and clients decreases.

\item
We propose \textbf{FedIL}, a novel FSSL framework that can enable combined with any type of semi-supervised model based on pseudo labeling and enables learning server-client consistency between supervised and unsupervised tasks. The experiment results show that FedIL outperforms most of the state-of-the-art FSSL baselines on both IID and non-IID settings.

\end{itemize}

\section{Related Work}
\subsection{Federated Learning}
Traditional Federated Learning (FL) is a distributed training technique for acquiring knowledge from decentralized clients without transmitting raw data to the server \cite{mcmahan2017communication}. A number of methods for averaging local weights at the server have been developed in recent years. FedAvg \cite{mcmahan2017communication} is the standard algorithm for FL, which averages local weights to update the global weight according to the local training size. FedProx \cite{li2020federated} uniformly averages the local weights while clients execute proximal regularisation against the global weights. The aggregation policy introduced by PFNM \cite{yurochkin2019bayesian} makes use of Bayesian non-parametric methods. FedMA \cite{wang2020federated} matches the hidden elements with comparable feature extraction signatures when averaging local weights. In addition to the research of aggregation strategies, Non-IID data is also one of the most significant problems of FL \cite{wang2020federated}, causing weight divergence and performance degradation, as explained in \cite{zhao2018federated}. Numerous solutions to this problem have been proposed, including providing a public dataset \cite{zhao2018federated}, knowledge distillation \cite{zhuang2020performance}, and regularising client training \cite{li2020federated}.

\begin{figure*}[t]
  \centering
  \includegraphics[width=0.9\textwidth]{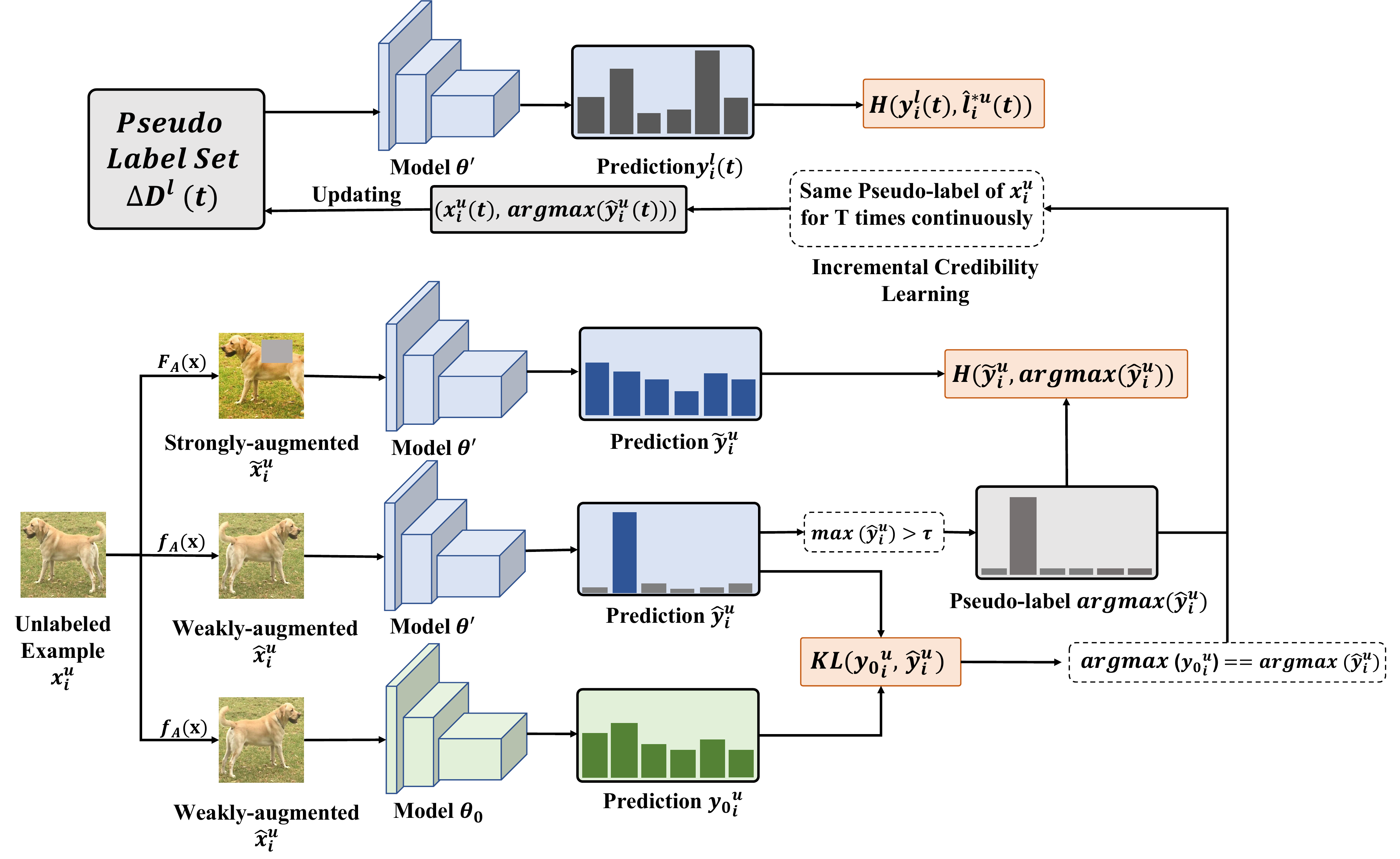}
  \caption{\textbf{Illustrative Running Example of Client.} The client uses the global model $\theta^{'}$ and the server model $\theta_0$ to train the unlabeled data to get the client model $\theta^{*}$. It comprises an end-to-end training pipeline with three steps: 1) Each client selects pseudo-labels with higher confidence. 2) Move the data with pseudo-labels that are continuous and stable for a certain time to the pseudo-label set in each client. 3) Use KL loss during client training to enforce the prediction consistency between the predictions made by clients and the server.}
  \label{Client}
\end{figure*}

\subsection{Semi-supervised Learning}
Semi-supervised learning (SSL) combines supervised learning and unsupervised learning, which aims to improve model performance by leveraging unlabeled data \cite{rasmus2015semi,zhou2005tri,chapelle2009semi}. 
The majority of SSL methods fall into two categories: Consistency Regularization and Self-Supervised Learning. 

Consistency regularization \cite{sajjadi2016regularization} assumes that transformation of the input instances will have no effect on the class semantics and requires that the model output be consistent across input perturbations. UDA \cite{xie2020unsupervised} and ReMixMatch \cite{DBLP:conf/iclr/BerthelotCCKSZR20} use data augmentations to enforce the consistency between the representations of two versions of unlabeled data. Pseudo-Label \cite{lee2013pseudo} is a popular and mainstream semi-supervised technique that use hard (1-hot) labels from the model's prediction as labels on unlabeled data, implicitly minimizing the prediction's entropy. FixMatch \cite{DBLP:conf/nips/SohnBCZZRCKL20} selects pseudo-labels as consistency between weak-strong augmented pairs from unlabeled data using a fixed and unified threshold. 

Self-supervised learning obtains supervisory signals from the data itself. For example, contrastive learning\cite{hadsell2006dimensionality,oord2018representation} tries to reduce the similarity of positive samples while increasing the similarity of negative samples. The negative pairs are created from either a memory bank, such as MoCo \cite{he2020momentum}, or a huge batch size, such as SimCLR \cite{chen2020simple}, but these methods strongly rely on the computation resources. Methods such as BYOL \cite{grill2020bootstrap} and SimSiam \cite{chen2021exploring}  can avoid negative pairs and only compare positive ones, which can be used when computation resources are limited. After a pre-train model is generated by self-supervised learning, a small number of labeled data will be fed to this model.

\subsection{Federated Semi-supervised Learning}
The majority of existing FL research focuses on supervised learning problems using ground-truth labels provided by clients while most clients are unlikely to be specialists in many real-world tasks, which is faced as an issue in recent studies \cite{jin2020towards}.  
%Learning representations from unlabeled decentralized data while ensuring data privacy is also an emerging area. Federated semi-supervised learning (FSSL) introduces unlabeled data into federated learning and tries to overcome the labeling issue in traditional FL. 
FSSL introduces learning representations from unlabeled decentralized data into FL and tries to overcome the labeling issue in traditional FL. 
Several solutions have been offered to realize FSSL by incorporating classical semi-supervised learning into the framework of federated learning. 

There are currently two FSSL scenarios, one is Labels-at-Client, and the other one is Labels-at-Server. For the Labels-at-Client, RSCFed \cite{liang2022rscfed} relies on label clients to assist in learning unlabeled data in other clients, while there is still a high risk of data leakage in this scenario. For Labels-at-Server, FedMatch \cite{jeong2021federated} focused on adapting the semi-supervised models to federated learning, which introduced the inter-client consistency that aims to maximize the agreement across models trained at different clients. FedU \cite{DBLP:conf/iccv/ZhuangG0ZY21} is based on the self-supervised method BYOL \cite{grill2020bootstrap}, which aims for representation learning. FedEMA \cite{zhuang2021divergence} is an upgraded version of FedU, which adaptively updates online networks of clients with EMA of the global model. Orchestra \cite{DBLP:conf/icml/LubanaTKDM22} relies on high representational similarity for related samples and low similarity across different samples.

Our FedIL consists of novel models that enable clients with just unlabeled data to complete semi-supervised training and the novel Cosine Similarity  with normalization based aggregation that accelerates the convergence of the global model in federated learning.

\section{Method}
In the following sections, we introduce FedIL, our proposed federated incremental learning approach shown in Figure \ref{Client} and Figure \ref{Server}. Section A will describe the scenario setting of Federated Semi-Supervised Learning. In the following sections B, C, and D, we will introduce the walkthrough procedures of the Client system, which has been illustrated as shown in Figure \ref{Client}. In section E, we will introduce training and weight selection of the Server system shown in Figure \ref{Server}. The summary of the whole system will be introduced in section F.

\subsection{Modeling and Federated Mapping}
Let the whole dataset for the training be $D$ as
\begin{equation}
    D=\left[X,y\right]=
    \begin{bmatrix}
X^s & y^s\\
X^u & y^u
\end{bmatrix}=
\begin{bmatrix}
D^s\\
D^u
\end{bmatrix}
\end{equation}
where $X$ refers to the input and $y$ refers to the label, $X^s$ refers to the input with known labels $y^s$, being the dataset $D^s=[X^s,y^s]$; while $X^u$ refers to the input without labels, and $y^u$ serves as unknown labels corresponding to $X^u$. 
The main task is to obtain the estimation $\bar{y}^u$ corresponding to the input $X^u$ with the mapping $f$ built based on $D$, which satisfies
\begin{equation}
\label{mappingFormula}
\left\lbrace
\begin{array}{ll}
    \min_f \|y^u-\bar{y}^u\|\\
    \bar{y}^u=f(X^u)
\end{array}
\right.
\end{equation}
in which $\|\cdot\|$ refers to the difference measurement.

Solving the mapping $f$ that satisfies~\eqref{mappingFormula} is the process of finding a function that can be fitted to the best mapping relationship by known ${D^s}$.
However, in the practical process, the unidentified label $y^u$ is unknown, so we cannot know exactly what the label corresponding to $X^u$ is. Therefore, we introduce a semi-supervised learning strategy for solving this mapping, which can be formulated as
\begin{equation}
\label{bary}
\bar{y}^u=f(X^u,\theta)
\end{equation}
in which $\theta$ refers to the weights for the learning.

\begin{figure}[t]
  \centering
  \includegraphics[width=0.48\textwidth]{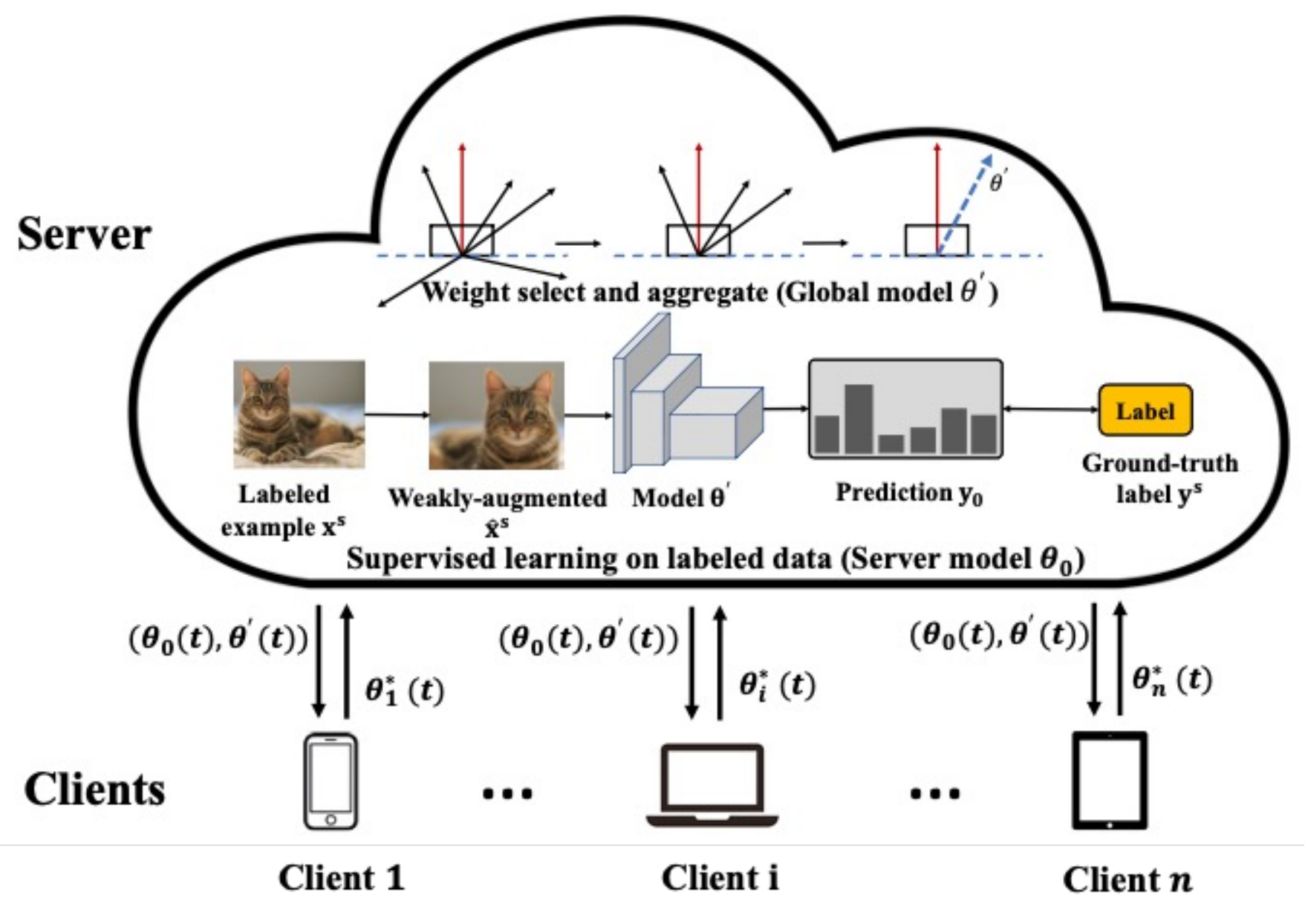}
  \caption{\textbf{Illustrative Running Example of Server.} We describe training and updating in the server and introduce the communication procedure between the server and local clients. The server uses Cosine Similarity with normalization to select client weights that are close to the server weight, which accelerates model training. The aggregated global model is $\theta^{'}$, and the server uses $\theta^{'}$ to supervise training the labeled data to get the server model $\theta_0$. }
  \label{Server}
\end{figure}

In federated learning, the learning system will take place in a distributed system that is composed of a Server and 
%organized in a decentralized fashion by 
multiple Clients with index $i=1,2,\ldots,n$. Each client uses a dataset $D_i^u$ with unlabeled data $X_i^u$.
%, as  correspondingly divided into groups $D_i^s, X_i^s$.
The server and all the clients have the same learning model $f$. 
Each local client learns a part of the data and submits its learning weights $\theta_i$ to the server, and the server (index $i=0$) collects and integrates weights, and then performs supervised learning on labeled dataset $D^s$.
%the server (index $i=0$) collects and integrates weights to form the overall learning result $\theta^'$ as the optimal solution of global training, which will be introduced in section 3.6.
The client-mapping corresponding to \eqref{bary} can be formulated as
\begin{equation}
%\label{yest}
\bar{y}_i^u = f({X}_i^u, \theta_i) \ \ \ \ \ \ i=1,2,\ldots,n
\end{equation}
and the optimal overall solution is obtained by
\begin{equation}
y^*=f(X,\theta^*)
\end{equation}
It can be seen that the main task now has been converted into finding the optimal parameter $\theta^*$ corresponding to~\eqref{mappingFormula} as
\begin{equation}
\label{thetaSol}
    \theta^* = \arg \min_{\theta} \|f(X,\theta)- y\|\\
\end{equation}
Once the $\theta^*$ is obtained the mapping $f$ can be determined.

Since $y^u \in y$ is unknown, it is impossible to solve $\theta$ by directly calculating the measure between the predicted label and the actual label of the domain, so we propose an iterative method to estimate $\theta^*$ based on $\bar{y}_i^u$ and $\theta_i$ with incremental similarity belief. 

\subsection{Hyper-Pseudo Pattern Space}
For the dataset $D^u=[X^u, y^u]$ with unknown labels, it is not feasible to use the difference between $\bar{y}^u$ and $y^u$ to construct a loss function for learning model updates. The $y^u$ are pseudo labels generated by the model strategies, which cannot be equivalent to the ground-truth labels and suffers from certain labeling errors. Therefore, we constructed the Hyper-Pseudo Pattern Space to approximate a pattern of loss model in FSSL, which is equivalent to the loss model for centralized learning built with $y^u$ and $\bar{y}^u$, through a set of nonlinear projections based on a set of augmentations of $X^u$ and corresponding estimations to the labels.

We expect to detect more features from the data itself to guess the commonality of their labels. Therefore, we introduce pseudo-data derivation from data and weights to generate hyper context. Specifically, we introduce three derivative mappings, including weak augmentation $f_A(\cdot)$, strong augmentation $F_A(\cdot)$, and %pseudo-optimal mapping 
the server reference $f(\cdot,\theta_0)$ to derive three estimations to the labels, as shown in Figure \ref{Client}, in which weakly and strongly augmentation sub-data of $X^u$ can be obtained by
\begin{equation}
\hat{X}^u = f_A(X^u) \ \ \ \ \ \ \ \widetilde{X}^u = F_A(X^u)
\end{equation}
and the estimated $\hat{y}^u, \widetilde{y}^u$ corresponding to $\hat{X}^u, \widetilde{X}^u$ are obtained by mapping $f$ as
\begin{equation}
\label{yest}
\hat{y}^u = f(\hat{X}^u, \theta^{'}) \ \ \ \ \ \ \ \widetilde{y}^u =f(\widetilde{X}^u, \theta^{'})
\end{equation}
and the server reference label 
%the pseudo-optimal label 
can be obtained by
\begin{equation}
\label{pseudoOptimal}
y_0^u = f(\hat{X}^u, \theta_0)
\end{equation}
where $\theta_0$ refers to the weight from supervised learning in the server.
%, and the mapping $f$ with parameter $\theta$ transmitted from the server is the aggregated weight in FSSL.
%where $\theta'$ refers to the aggregated weight in the federated learning and the mapping $f$ with parameter $\theta$ is solved by semi-supervised learning.
The client-mapping corresponding to \eqref{yest}-\eqref{pseudoOptimal} can be formulated as
\begin{equation}
\begin{aligned}
\hat{y}_i^u = f(\hat{X}_i^u, \theta^{'})  \ \ \ \widetilde{y}_i^u =f(\widetilde{X}_i^u, \theta^{'})  \ \ \ {y_0}_i^u = f(\hat{X}_i^u, \theta_0) \\ i=1,2,\ldots,n
\end{aligned}
\end{equation}
With the estimated labels, we can select a part of the estimated $y^u$ (pseudo labels) as reference labels to enhance the likelihood of classification learning for unlabeled samples. To distinguish the unselected $X^u, y^u$, we rewrite the selected pseudo-labels and their corresponding dataset as $D^l=[X^l,\hat{l}^{*u}]$, which are described in section D.

\subsection{Iterative Similarity Fusion}
Since there is no %definite 
known label $y^u$ to determine the final optimal solution $f$ in~\eqref{mappingFormula}, we design an iterative approach to optimize the processing of 
%to gradually approach 
the most probable optimal solution. On the basis of the above hyper-pseudo space, we use a number of different similarity comparisons and fusions for final confirmation and find an optimal solution that can optimize the fitting weights in different hyper-pseudo contexts.

Let $\theta_i(t)$ be the parameter of $i$th local client at $t$ training rounds. The local optimal parameter $\theta^*_i(t)$ is formulated as
\begin{equation}
\label{ISF}
    \theta_i^*(t) = f^b(\xi_i^a(t) + \xi_i^b(t))
\end{equation}
where $f^b$ refers to the back-propagation weight updating for $f$, $\xi^a$ is the Cross-entropy of the 
$\widetilde{y}_i^u$ and $\hat{y}_i^u$, and $\xi^b$ is the Kullback–Leibler divergence to measure how one probability distribution $\hat{y}_i^u$ is different from a reference probability distribution ${y_0}_i^u$, which can be formulated as
\begin{equation}
\left\lbrace
\begin{array}{ll}
     \xi_i^a(t)= \varepsilon(\max(\hat{y}_i^u(t))-\tau) H(\widetilde{y}_i^u(t),\arg\max(\hat{y}_i^u(t)))\\
\xi_i^b(t)=KL({y_0}_i^u(t),\hat{y}_i^u(t))
\end{array}
\right.
\end{equation}
in which $\varepsilon(x)$ refers to step function as
\begin{equation}
\label{stepFun}
\varepsilon (x)=
   \left\lbrace
   \begin{array}{ll}
1 \ \ x\geq 0\\
0 \ \ else
   \end{array}
   \right.
\end{equation}
$f^b$ here is equivalent to fusing the features obtained by each context and using it to update the weights of the clients.

In the standard semi-supervised learning methods, learning on labeled and unlabeled data is simultaneously done using a shared set of weights. However, this may result in forgetting knowledge of labeled data in disjoint learning scenarios, such as FSSL. FedIL ensures consistency between the class with the highest probability predicted by the local model and the server model when updating the pseudo-label dataset. However, FedIL wants to learn information about all classes instead of treating all negative labels uniformly when training a local model with an unselected unlabeled dataset. In the output of the softmax layer, other classes carry a lot of information besides the class with the highest probability. Therefore, in addition to the regular Cross-entropy loss $\xi^a(t)$ in semi-supervised models, we propose a consistency KL loss $\xi^b(t)$ that regularizes the models learned between the server and clients to output the same estimation.

%In the standard semi-supervised learning methods, learning on labeled and unlabeled data is simultaneously done using a shared set of weights. However, this may result in forgetting knowledge of labeled data in disjoint learning scenarios, such as FSSL. Therefore, in addition to the regular Cross-entropy loss $\xi^a(t)$ in semi-supervised models, we propose a consistency KL loss $\xi^b(t)$ that regularizes the models learned between the server and clients to output the same estimation.

\subsection{Incremental Credibility Learning with Dataset Updating}
For unlabeled training data, it is difficult to guarantee that pseudo-labels given by the system at the beginning are correct. We, therefore, design incremental confidences to observe whether their estimated labels always stabilize at a deterministic result after successive training and validation. Each continuous identification of a classification result will increase the credibility of its label. In the current FSSL scenario, labeled data and unlabeled data are disjoint. To enable complete semi-supervised training in the client, we design a dataset updating strategy by introducing a pseudo-label set to store the selected high-confidence data.

% Similar to \eqref{ISF}, the similarity fusion with incremental credibility updating can be formulated as
% \begin{equation}
%     \theta_i^*(t) = f^b(\alpha(\hat{y}_i^u(t)) \xi_i^a(t) + \xi_i^b(t))
% \end{equation}
% where the incremental credibility $\alpha(\hat{y}_i^u(t))$ can be formulated as
% \begin{equation}
%  \alpha(\hat{y}_i^u(t))= \prod_{k=0}^T \varepsilon (\max(\hat{y}_i^u(t-t^k_i))-\tau)
%   % \alpha(\hat{y}_i^s(t))= \delta \left(\sum_t^T \varepsilon( max(\hat{y}_i^s(t)) - \tau(t)) - T \right)
% \end{equation}
% in which $t^k_i$ refers to the time span between the $k$th client-activation time prior to time $t$, e.g., $k=0$ refers to now, $k=1$ refers to the time span between now and last time when client $i$ was activated; $k$ refers to the time span when the $k$th client-activation moment previous to time $t$. $\tau$ refers to the threshold of the score of classification,
% and $\varepsilon(x)$ refers to step function defined in~\eqref{stepFun}.
% In this paper, we set these two hyper-parameters as $\tau = 0.95, T=7$.

In order to integrate the credibility into the optimization for the parameter mapping solution, we further formulate the optimal incremental credibility learning as
% \begin{equation}
%     \theta_i^*(t) = f^b(\alpha(\hat{y}_i^u(t)) \xi_i^a(t) + \xi_i^b(t) + \xi_i^c(t))
% \end{equation}
\begin{equation}
    \theta_i^*(t) = f^b( \xi_i^a(t) + \xi_i^b(t) + \xi_i^c(t))
\end{equation}
in which $f^b(\cdot)$ refers to the back-propagation mapping to update the weight based on the loss input $(\cdot)$ and
\begin{equation}
    \xi^{c}(t)=H(y_{i}^{l}(t),\hat{l}_{i}^{*u}(t))
\end{equation}
%{{\xi}^c}(t)=H({{y_i}^l}(t),{{{\hat{l}_i}^{*}}^u}(t))
where ${y}_i^l(t)$ refers to the predictions generated by the input $X^l_i$ using the model trained by $D_i^l(t)$, which can be formulated as
\begin{equation}
    {y}_i^l(t)=f(X^l_i,\theta_i(t))%\arg\max(\hat{y}_i^s(t))
\end{equation}
and $\hat{l}_i^{*u}(t)$ refers to the optimal pseudo-labels which satisfied two main principles, 1) the label has been consecutively selected as a pseudo-label no less than $T$ times; 2) the count of the compliance of the rule $\arg\max({y_0}_i^u(t))==\arg\max(\hat{y}_i^u(t)))$ has been no less than $T$ times.
The first time an unlabeled image is involved in training in $D_i^l(t)$ can be formulated as
%which can be formulated as
\begin{equation}
    \hat{l}_i^{*u}(t)= \alpha (\hat{y}_i^u(t), \check{l}_i^u(t))
    \arg \max(\hat{y}_i^u(t-t^1_i))
%    \check{l}_i^u(t)
\end{equation}
\begin{equation}
    \check{l}_i^u(t)= {\mathbbm{1}} (\arg \max({y_0}_i^u(t))==\arg\max(\hat{y}_i^u(t)))
\end{equation}
in which $ {\mathbbm{1}}$ refers to the selection function and
\begin{equation}
 \alpha(\hat{y}_i^u(t), \check{l}_i^u(t))= \prod_{k=1}^T \varepsilon (\max(\hat{y}_i^u(t-t^k_i))-\tau)
 \varepsilon (\check{l}_i^u(t-t^k_i)))
  % \alpha(\hat{y}_i^s(t))= \delta \left(\sum_t^T \varepsilon( max(\hat{y}_i^s(t)) - \tau(t)) - T \right)
\end{equation}

When the result is continuous and stable for a long time after a series of confidence increments accumulate to reach a threshold, the data $X_i^u(t)$ and the label $\arg \max(\hat{y}_i^u(t))$ will be used as candidate credible samples as $\Delta D_i^l(t)$ and incorporated into the known label dataset $D_i^l(t+1)$ as a reference with unchanged labels for the following training rounds. 
%Conversely, if the candidate data's label is recognized to the other label result several consecutive times, the original label in the $D_i^l(t)$ will be updated.

%Conversely, if the classification result is not recognized continuously, it will also increase the unreliability of its label. When the candidate data and its labels are not recognized several consecutive times, they will be downgraded back to the unlabeled dataset $\Delta D^u(t)$, and removed from the candidate label set $D^l(t)$ when updating.

The federated pseudo-label set updating can be formulated as
\begin{equation}
D_i^l(t+1) = [X^l_i(t+1),\hat{l}_i^{*u}(t+1)]\\
\end{equation}
where
\begin{equation}
   \left\lbrace
   \begin{array}{rl}
   D_i^l(t)&=\{X_i^l(t),\hat{l}_i^{*u}(t)\}\\
    X_i^l(t+1)&=X_i^l(t)\cup \Delta X_i^l(t)\\
    \Delta X_i^l(t)&=\{F_A^{-1}(\widetilde{y}_i^u(t),\theta_i(t))\ |
%    \Delta X_i^l(t)&=\{f_A^{-1}(f^{-1}(\hat{y}_i^u,\theta_i))\ |
    \ \alpha(\hat{y}_i^u(t), \check{l}_i^u(t))=1\}\\
    \hat{l}_i^{*u}(t+1)&=\hat{l}_i^{*u}(t)\cup\Delta\hat{l}_i^{*u}(t) \\
    \Delta\hat{l}_i^{*u}(t)&=f(\Delta X_i^l(t),\theta_i(t))\\
   \end{array}
   \right.
\end{equation}
in which $F_A^{-1}$ refer to the inverse mapping of $F_A$.

\subsection{Global Incremental Learning}
As an iterative update of federated learning, we propose global incremental learning by \textbf{Cosine Similarity} with weight normalization to update the learning weights of each round which is shown in Figure \ref{Server}, that is, the final learning results of each round will be aggregated into $\theta'$ through the server ($i=0$) system and assigned to each client as the weights of the next round of initial training.
The parameter updating of the server ($i=0$) can be formulated in an iterative form as
\begin{equation}
   \theta_0(t)=f^b(H(y^s, \arg\max(y_0(t))))
\end{equation}
where $y_0(t)$ refers to the prediction based on the $X^s$ with the model $f$ and parameter $\theta'$, i.e.,
\begin{equation}
   y_0(t)=f(X^s,\theta'(t))
\end{equation}
in which $D^s=[X^s,y^s]$ refers to the supervised learning dataset with known labels $y^s$.
So the process of solving the model $f$ is equivalent to the process of optimizing $\theta'(t)$.

The aggregated updating can be formulated as
\begin{equation}
\label{GIL}
    \theta'(t+1) = \theta'(t)+\Delta \theta(t)
\end{equation}
where the $\Delta \theta(t)$ can obtained by \eqref{incrementalTheta}.

The incremental $\Delta \theta(t)$ can be formulated as
\begin{equation}
\label{incrementalTheta}
\Delta \theta(t) = \frac{\sum^n_{i=1} \varepsilon (S(t)) (\theta_i^*(t)-\theta'(t))}{\sum_{i=1}^n\varepsilon (S(t))}
\end{equation}
in which $S(t)$ refers to the cosine similarity between $\theta^*_i(t)-\theta'(t)$ and $\theta_0(t)-\theta'(t)$ that
\begin{equation}
    S(t)=\cos(\theta^*_i(t)-\theta'(t),\theta_0(t)-\theta'(t))
\end{equation}

\subsection{Summarization of the Framework}
%We finally present the pseudo-code of the FedIL algorithm in the supplementary materials. 
FedIL well solves the problem of disjoint between labeled data and unlabeled data and achieves complete semi-supervised training by establishing its own pseudo-label set on each client. Furthermore, Global Incremental Learning helps the server to effectively select the uploaded client weights, which accelerates the convergence of the global model. Based on the federated learning scenario, we separately demonstrate the training process on the server in Algorithm \ref{alg1} and the training process on local clients in Algorithm \ref{alg2}. 

For Server-side training in Algorithm \ref{alg1}, we aggregate the selected weights from clients in the same direction as the server weight by cosine similarity with normalization. For example, we demonstrate how to derive $\theta^{'}(t=2)$. Let $t=1$, which means the first training round starts, and initialize global model $\theta^{'}(t=1)$ randomly. Take $\theta^{'}(t=1)$ as a reference point and train on the labeled data to get $\theta_0(t=1)$, and broadcast $\theta^{'}(t=1)$ and $\theta_0(t=1)$ to clients for training. After the first training round, select the client models $\theta^{*}_i(t=1)$ that are close to the server model $\theta_0(t=1)$ by using cosine similarity $\cos(\theta^{*}_i(t=1)-\theta^{'}(t=1),\theta_0(t=1)-\theta^{'}(t=1))$. Next, calculate the increment of the difference $\Delta \theta(t=1)$ between the selected client models and the global model by (25), and use (24) to get $\theta^{'}(t=2)$.

For Client-side training in Algorithm \ref{alg1}, we use KL loss to enforce Server-Client consistency and use Incremental Credibility Learning to build high-confidence Pseudo Label Sets.
For example, when selected clients receive the server weight $\theta_0(t)$ and the global weight $\theta^{'}(t)$ from the server, unsupervised training based on pseudo labeling will start on clients and we use Cross-entropy loss and KL loss to update Client weight $\theta^*_i(t)$. For Pseudo Label Set updating, when pseudo labels $\arg\max(\hat{y}_i^u(t))$ of one image are continuous and stable for a long period of time, the data $X_i^u(t)$ and the label $\arg \max(\hat{y}_i^u(t))$ will be used as candidate credible samples in $\Delta D_i^l(t)$ and incorporated into the known label dataset $D_i^l(t+1)$ as a reference with unchanged labels for the subsequent training rounds.

\begin{algorithm}
    \caption{FedIL in the Server}%标题
    \label{alg1}%标签
\begin{algorithmic}[1]
    \FOR{each training round $t = 1,2,3...$}
        \IF{$t=1$} 
            \STATE randomly initialize $\theta^{'}(t=1)$
            \STATE \textbf{Processing with} (22) to update Server weight $\theta_0(t=1)$ 
            \STATE Randomly select 5 clients from 100 clients
            \STATE broadcast $\theta^{'}(t=1)$ and $\theta_0(t=1)$ to the next selected Clients
        \ELSE
            \STATE Received $\theta^*_i(t)$ from clients \\
            \COMMENT{\textit{received the weights from clients selected in the last round.}}
            \STATE \textbf{Processing with} (24)(25)(26) to update the global weight $\theta^{'}(t+1)$
            %\Comment{\textit{Prediction distribution of $x^t_n$ }}
            \STATE  \textbf{Processing with} (22)(23) to update Server weight $\theta_0(t+1)$
            \STATE Randomly select 5 clients from 100 clients
            \STATE broadcast $\theta^{'}(t+1)$ and $\theta_0(t+1)$ to the next selected Clients
        \ENDIF
    \ENDFOR
\end{algorithmic}
\end{algorithm}

\begin{algorithm}
    \caption{FedIL in selected Clients}%标题
    \label{alg2}%标签
\begin{algorithmic}[1]
    \FOR{each selected Client $i$ in parallel}
        \STATE Received the Server weight $\theta_0(t)$ and the global weight $\theta^{'}(t)$ from the Server
        \STATE \textbf{Processing with} (12)(13)(14)(15)(16) to update Client weight $\theta^*_i(t)$ 
        \FOR{each Pseudo Label Set updating}
            \STATE \textbf{Processing with} (17)(18)(19)(20)(21) to update Pseudo Label Set $D_i^l(t+1)$
        \ENDFOR
        \STATE upload $\theta^*_i(t)$ to the Server
    \ENDFOR
\end{algorithmic}
\end{algorithm}

\section{Convergence Analysis of FedIL}

%The proposed global incremental learning possesses the mathematical property of fixed point convergence, which can be proved as following.

In this section, we provide that the proposed global incremental learning possesses the theoretical justification of fixed point convergence, which can be proved as follows.

\begin{definition}
\label{defContractMapping}
Let $(X,d)$ be a complete metric space. Then a map $ T : X \to X$ is a contraction mapping on $X$ if $\exists q\in [0,1)$ such that
\begin{equation}
    d(T(x),T(y))\leq qd(x,y) \forall x,y\in X.
\end{equation}
\end{definition}

\begin{lemma}
\label{banachTh}
Banach Fixed Point Theorem. Let $(X,d)$ be a non-empty complete metric space with a contraction mapping $T:X\to X$. Then T admits a unique fixed-point $x^{*}$ in $X$ (i.e. $T(x^{*})=x^{*})$.
\end{lemma}

\begin{remark}
\normalfont{The fixed point $x^{*}$ in the non-empty complete metric space $(X,d)$ can be found as 
\begin{equation}
    \lim _{n\to \infty }x_{n}=\lim _{n\to \infty }T^n(x_{0})=x^{*}
\end{equation}
in which $x$ starts with an arbitrary element $x_{0}\in X$ and define a sequence with the contraction mapping as $(x_{n})_{n\in {\mathbb  N}}$ by $ x_{n}=T(x_{n-1})$ for $n\geq 1$.}
\end{remark}

\begin{theorem}
\label{THDtheta}
Let $(\Theta,d)$ be a non-empty norm space of parameter set, in which
$\Theta=\{\theta\},d(\theta)=\|\theta\|$.
Under the proposed Global Incremental Learning~\eqref{GIL}~\eqref{incrementalTheta}, $\exists \theta^* \in \Theta$ that
\begin{equation}
\lim_{t\to \infty }\theta'(t)=\theta^{*}
\end{equation}
if and only if the norm of $\Delta \theta(t)$ is monotonically decreasing with $t$ that
\begin{equation}
\label{dtheta}
    \frac{d \| \Delta \theta(t) \|}{dt} \leq 0
\end{equation}
\end{theorem}

\begin{proof}
%(Brief Steps) 
\normalfont{In $(\Theta,d)$, define mapping $T: \Theta \to \Theta$ as
\begin{equation}
\label{defT}
T(\theta'(t))=\theta'(t)+\Delta \theta(t)=\theta'(t+1)
\end{equation}
where $\Delta \theta(t)$ is obtained by~\eqref{incrementalTheta}.
Therefore, 
\begin{equation}
d(T(\theta'(t)),T(\theta'(t-1)))=\|\theta'(t+1)-\theta'(t)\|=\|\Delta \theta(t)\|
\end{equation}
similarly, $d(\theta'(t),\theta'(t-1))=\|\Delta \theta(t-1)\|$.
\begin{equation}
\frac{d(T(\theta'(t)),T(\theta'(t-1)))}{d(\theta'(t),\theta'(t-1))}=\frac{\|\Delta \theta(t)\|}{\|\Delta \theta(t-1)\|}
\end{equation}
Since
\begin{equation}
\frac{d \| \Delta \theta(t) \|}{dt} \leq 0
\end{equation}
we have $\|\Delta \theta(t)\| \leq \|\Delta \theta(t-1)\|$ therefore, $\exists q\in(0,1)$
\begin{equation}
\label{banachdT}
d(T(\theta'(t)),T(\theta'(t-1)))\leq q d(\theta'(t),\theta'(t-1))
\end{equation}
where
\begin{equation}
\label{qRef}
     \frac{\|\Delta \theta(t)\|}{\|\Delta \theta(t-1)\|} \leq q <1
\end{equation}
Therefore, with Definition \ref{defContractMapping}, $T$ defined by \eqref{defT} is a contraction mapping, so the sufficiency of Theorem~\ref{THDtheta} is proved with
lemma~\ref{banachTh}.
When~\eqref{banachdT} holds, if
\begin{equation}
\label{dthetaGre}
    \frac{d \| \Delta \theta(t) \|}{dt} > 0
\end{equation}
which yields
\begin{equation}
\frac{\|\Delta \theta(t)\|}{\|\Delta \theta(t-1)\|}=
\frac{d(T(\theta'(t)),T(\theta'(t-1))}{d(\theta'(t),\theta'(t-1))}>1
\end{equation}
in which it is impossible to find $q\in (0,1)$ to satisfy~\eqref{banachdT}, which shows the contradiction and the necessity is proved with proof by contradiction.
Therefore, Theorem~\ref{THDtheta} holds if and only if the norm of $\Delta \theta(t)$ is monotonically decreasing with $t$.}
\end{proof}

Theorem~\ref{THDtheta} reveals the relationship among global learning $\theta'(t)$, local learning $\theta_i^*(t)$, server learning $\theta_0(t)$ and gives the necessary and sufficient condition for federated learning to enter the convergence stage. It is when the incremental difference $\Delta \theta(t)$ between the distributed local learning weights and the server weights decreases according to the norm $\|\Delta \theta(t)\|$, the global incremental learning weights $\theta'(t)$ gradually converge to the fixed point $\theta^*$ in the parameter space $\Theta$. 

The condition of Theorem 1 is only $\Delta\theta$ (weight different) monotonically decreasing for each training round, and the conclusion is $\Delta\theta$ approaching 0 as time goes on. The novelty of the Theorem is the new definition of convergence for FSSL that is based on weight difference rather than the loss in supervised FL. This is because client data are unlabeled in FSSL, and the loss is computed by pseudo-labels rather than ground-truth labels, which is unreliable and thus cannot guarantee model convergence. 
We are the first to use weight difference as a criterion rather than loss to determine model convergence in FSSL.

\begin{table*}[!t]
\caption{Method-Wise Accuracy Evaluation Stats on IID and non-IID settings of MNIST, CIFAR10 and CIFAR100 datasets. Our proposed FedIL outperforms most of the state-of-the-art methods.}
\centering
\small
\renewcommand\tabcolsep{6.0pt}
\begin{tabular}{lcccccccccc}
\hline
                 & \multicolumn{4}{c}{MNIST}                                                                                                                                                                                                                                                                               & \multicolumn{4}{c}{CIFAR10}                                                                                                                                                                                                                                                                             & \multicolumn{2}{c}{CIFAR100}                                                                                                        \\ \cline{2-11} 
                 & \multicolumn{2}{c}{IID}                                                                                                             & \multicolumn{2}{c}{non-IID}                                                                                                                                        & \multicolumn{2}{c}{IID}                                                                                                                                           & \multicolumn{2}{c}{non-IID}                                                                                                          & IID                                                              & non-IID                                                           \\ \hline
Label rate       & $\gamma$=0.01                                                           & $\gamma$=0.1                                                            & $\gamma$=0.01                                                                          & $\gamma$=0.1                                                                           & $\gamma$=0.01                                                                          & $\gamma$=0.1                                                                           & $\gamma$=0.01                                                           & $\gamma$=0.1                                                            & $\gamma$=0.1                                                            & $\gamma$=0.1                                                            \\ \hline
Fully-Supervised & \multicolumn{4}{c}{99.50\%±0.02}                                                                                                                                                                                                                                                                        & \multicolumn{4}{c}{93.59\%±0.03}                                                                                                                                                                                                                                                                        & \multicolumn{2}{c}{71.71\%±0.05}                                                                                                    \\ \hline
FedMatch         & \begin{tabular}[c]{@{}c@{}}97.13\%\\ ±0.15\end{tabular}          & \begin{tabular}[c]{@{}c@{}}98.89\%\\ ±0.21\end{tabular}          & \begin{tabular}[c]{@{}c@{}}97.22\%\\ ±0.23\end{tabular}                         & \begin{tabular}[c]{@{}c@{}}98.28\%\\ ±0.26\end{tabular}                         & \begin{tabular}[c]{@{}c@{}}55.30\%\\ ±0.17\end{tabular}                         & \begin{tabular}[c]{@{}c@{}}82.20\%\\ ±0.29\end{tabular}                         & \begin{tabular}[c]{@{}c@{}}54.26\%\\ ±0.13\end{tabular}          & \textbf{\begin{tabular}[c]{@{}c@{}}79.50\%\\ ±0.15\end{tabular}} & \begin{tabular}[c]{@{}c@{}}46.43\%\\ ±0.33\end{tabular}          & \textbf{\begin{tabular}[c]{@{}c@{}}46.60\%\\ ±0.15\end{tabular}} \\ 
FedU             & \begin{tabular}[c]{@{}c@{}}95.42\%\\ ±0.23\end{tabular}          & \begin{tabular}[c]{@{}c@{}}98.13\%\\ ±0.12\end{tabular}          & \begin{tabular}[c]{@{}c@{}}96.74\%\\ ±0.22\end{tabular}                         & \begin{tabular}[c]{@{}c@{}}98.78\%\\ ±0.08\end{tabular}                         & \begin{tabular}[c]{@{}c@{}}54.17\%\\ ±0.08\end{tabular}                         & \begin{tabular}[c]{@{}c@{}}60.05\%\\ ±0.13\end{tabular}                         & \begin{tabular}[c]{@{}c@{}}58.73\%\\ ±0.15\end{tabular}          & \begin{tabular}[c]{@{}c@{}}72.36\%\\ ±0.32\end{tabular}          & \begin{tabular}[c]{@{}c@{}}30.82\%\\ ±0.26\end{tabular}          & \begin{tabular}[c]{@{}c@{}}31.51\%\\ ±0.15\end{tabular}          \\ 
FedEMA           & \begin{tabular}[c]{@{}c@{}}97.17\%\\ ±0.21\end{tabular}          & \begin{tabular}[c]{@{}c@{}}98.90\%\\ ±0.13\end{tabular}          & \begin{tabular}[c]{@{}c@{}}96.26\%\\ ±0.13\end{tabular}                         & \begin{tabular}[c]{@{}c@{}}98.60\%\\ ±0.25\end{tabular}                         & \begin{tabular}[c]{@{}c@{}}54.85\%\\ ±0.09\end{tabular}                         & \begin{tabular}[c]{@{}c@{}}63.73\%\\ ±0.32\end{tabular}                         & \begin{tabular}[c]{@{}c@{}}58.44\%\\ ±0.22\end{tabular}          & \begin{tabular}[c]{@{}c@{}}72.49\%\\ ±0.18\end{tabular}          & \begin{tabular}[c]{@{}c@{}}30.25\%\\ ±0.26\end{tabular}          & \begin{tabular}[c]{@{}c@{}}31.65\%\\ ±0.23\end{tabular}          \\ 
Orchestra        & \begin{tabular}[c]{@{}c@{}}96.57\%\\ ±0.21\end{tabular}          & \begin{tabular}[c]{@{}c@{}}97.86\%\\ ±0.13\end{tabular}          & \begin{tabular}[c]{@{}c@{}}95.97\%\\ ±0.28\end{tabular} & \begin{tabular}[c]{@{}c@{}}97.74\%\\ ±0.15\end{tabular} & \begin{tabular}[c]{@{}c@{}}60.32\%\\ ±0.17\end{tabular} & \begin{tabular}[c]{@{}c@{}}67.17\%\\ ±0.23\end{tabular} & \begin{tabular}[c]{@{}c@{}}57.92\%\\ ±0.25\end{tabular}          & \begin{tabular}[c]{@{}c@{}}65.24\%\\ ±0.17\end{tabular}          & \begin{tabular}[c]{@{}c@{}}31.15\%\\ ±0.17\end{tabular}          & \begin{tabular}[c]{@{}c@{}}31.42\%\\ ±0.26\end{tabular}          \\ \hline
\textbf{FedIL}            & \textbf{\begin{tabular}[c]{@{}c@{}}98.30\%\\ ±0.21\end{tabular}} & \textbf{\begin{tabular}[c]{@{}c@{}}99.05\%\\ ±0.13\end{tabular}} & \textbf{\begin{tabular}[c]{@{}c@{}}98.46\%\\ ±0.08\end{tabular}}                & \textbf{\begin{tabular}[c]{@{}c@{}}99.08\%\\ ±0.25\end{tabular}}                & \textbf{\begin{tabular}[c]{@{}c@{}}61.31\%\\ ±0.25\end{tabular}}                & \textbf{\begin{tabular}[c]{@{}c@{}}82.51\%\\ ±0.14\end{tabular}}                & \textbf{\begin{tabular}[c]{@{}c@{}}58.98\%\\ ±0.16\end{tabular}} & \begin{tabular}[c]{@{}c@{}}79.14\%\\ ±0.27\end{tabular}          & \textbf{\begin{tabular}[c]{@{}c@{}}48.25\%\\ ±0.17\end{tabular}} & \begin{tabular}[c]{@{}c@{}}46.39\%\\ ±0.25\end{tabular}          \\ \hline
\end{tabular}
\label{Acc}
\end{table*}

\section{Experiment}

\subsection{Experimental Setup}
\noindent\textbf{Datasets.}
%We conduct our experiments using three public datasets, including MNIST \cite{lecun1998gradient}, CIFAR10 and CIFAR100 \cite{krizhevsky2009}. The MNIST dataset is split into a 60,000-image training set and a 10,000-image test set. There are 60,000 32x32 color images in CIFAR10 and CIFAR100, where the training set has 50,000 images and the test set has 10,000 images. The datasets of MNIST and CIFAR10 are used for the task of 10-class image classification, while CIFAR100 is utilized for the 100-class image classification task.
We conduct our experiments using three public datasets, including MNIST \cite{lecun1998gradient}, CIFAR10 and CIFAR100 \cite{krizhevsky2009}. The MNIST dataset is a widely used dataset in the field of computer vision and machine learning. It consists of handwritten digit images and is considered a benchmark dataset for image classification tasks. This dataset is split into a 60,000-image training set and a 10,000-image test set. There are 60,000 32x32 color images in CIFAR10 and CIFAR100, where the training set has 50,000 images and the test set has 10,000 images. The CIFAR10 and CIFAR100 datasets, on the other hand, contain natural images belonging to 10 and 100 classes respectively. These datasets contain a large number of color images, making them a popular choice for testing the performance of image classification algorithms. Both CIFAR datasets have been used extensively in the research community to evaluate the performance of various computer vision and machine learning models.

\noindent\textbf{Federated System Setting.}
%The whole dataset $D$ will be shuffled randomly into two groups, labeled set $D^s$ and unlabeled set $D^u$. The server will hold the labeled dataset, while unlabeled data will be further distributed to $K = 100$ clients.
%Each client has $\tfrac{|D^u|}{K}$ instances for IID setting and 20\% of whole classes for non-IID setting, where $|D^u|$ is the quantity of unlabeled training data.  In this scenario, we make $\gamma$ represent the ratio of labeled data on the whole training dataset. Namely, there are $\gamma*|D|$ labeled samples in the server, and $\tfrac{(1-\gamma)*|D|}{K}$ unlabeled data in each client. 
%Our paper only consider \textbf{IID} setting, where both labeled data in the Server and unlabeled data in Clients have $C$ classes. 
The whole dataset $D$ will be randomly divided into two separate groups: a labeled set $D^s$ and an unlabeled set $D^u$. The labeled set will be held by the server, while the unlabeled set will be further divided and distributed among $K=100$ clients.
Each client will receive $\tfrac{|D^u|}{K}$ instances for the independent and identically distributed (IID) setting and 20\% of the total classes for the non-IID (Not IID) setting. The quantity of unlabeled training data is represented by $|D^u|$.
To represent the ratio of labeled data in the entire training dataset, the variable $\gamma$ is introduced. This means that there will be $\gamma * |D|$ labeled samples on the server and $\tfrac{(1-\gamma)*|D|}{K}$ unlabeled data on each client. This distribution is designed to ensure an equal distribution of labeled and unlabeled data among the clients and the server.
We set $\gamma=0.01, 0.1$ in the experiments, and only 5 local clients are working in each round. The total training rounds are 2000.

\noindent\textbf{Baselines.}
%To fairly evaluate the proposed FedIL framework, we use the same backbone ResNet9 and the following state-of-the-art benchmarks. \textbf{1) Fully-Supervised}: Centralized supervised training for the whole labeled datasets. \textbf{2) FedMatch \cite{jeong2021federated}}: FedAvg-FixMatch with inter-client consistency and parameter decomposition. \textbf{3) FedU \cite{DBLP:conf/iccv/ZhuangG0ZY21}}: Using the divergence-aware predictor updating with self-supervised BYOL \cite{grill2020bootstrap}. \textbf{4) FedEMA \cite{zhuang2021divergence}}:  Using EMA of global model to adaptively update online client networks. \textbf{5) Orchestra \cite{DBLP:conf/icml/LubanaTKDM22}}:  a novel clustering-based FSSL technique. 
To fairly evaluate the proposed FedIL framework, we use the same backbone ResNet9 and the following state-of-the-art benchmarks. \textbf{1) Fully-Supervised}: centralized supervised training for the whole labeled datasets.
\textbf{2) FedMatch} \cite{jeong2021federated}: use FedAvg-FixMatch with inter-client consistency and parameter decomposition to train the models.
\textbf{3) FedU} \cite{DBLP:conf/iccv/ZhuangG0ZY21}: use the divergence-aware predictor updating technique with self-supervised BYOL \cite{grill2020bootstrap} to train the models. 
\textbf{4) FedEMA} \cite{zhuang2021divergence}:  use an Exponential Moving Average (EMA) of the global model to adaptively update online client networks. 
\textbf{5) Orchestra} \cite{DBLP:conf/icml/LubanaTKDM22}: use a novel clustering-based FSSL technique.

\subsection{Experimental Results}

\noindent\textbf{Comparison with FSSL methods.} We demonstrate our experimental results in Table \ref{Acc}, which shows a performance summary using MNIST, CIFAR10 and CIFAR100 by different FSSL methods with the comparison of accuracy and parameters in the scenarios of IID and non-IID.

The method FedIL we use in our experiments is limited in the sense that it only utilizes the training dataset and the test dataset, and does not incorporate the validation dataset. This is in contrast to FedMatch, which leverages the validation dataset to optimize its weight aggregation by adjusting its parameters in the server. The validation dataset plays a crucial role in this process as it consists of labeled data, and hence allows for a supervised learning approach. By using the validation dataset, FedMatch effectively learns from two labeled datasets, namely $D^{s1}$ which is the labeled data in the training dataset, and $D^{s2}$ which is the labeled data in the validation dataset. This results in better performance compared to our method, which only utilizes a single labeled dataset.

%Since our method only uses the training dataset and the test dataset but not the validation dataset in the experiments, some results are slightly lower than FedMatch because FedMatch uses the validation dataset in the server to adjust the weights aggregation, and the validation dataset is a dataset with known labels. This type of solution is equivalent to learning with two labeled datasets, one with the labeled data $D^{s1}$ in the training dataset, and one conditioning with another $D^{s2}$ in the validation.

%In practice, this solution is not fully applicable, especially when the known labeled datasets are very limited, because there are not enough labeled samples to divide enough datasets for training and validation, and over-reliance on the limited label data will likely cause over-fitting.
%Aiming at this issue, our model parameter-solving process is completely based on incremental credibility, which is more in line with the practical application of semi-supervision learning with the reduction of the dataset requirements for learning. Meanwhile, it can also be extended to use with validation data, so it can have wider applicability.

In practice, this solution is not fully applicable, especially when the amount of labeled data is limited. In such cases, dividing the limited labeled data into separate datasets for training and validation becomes problematic as there may not be enough labeled samples to ensure a representative and balanced split. Over-fitting is a common issue in these scenarios, where the model becomes overly reliant on the limited labeled data and fails to generalize well to unseen data.
To address these limitations, our model employs a different approach to solving for parameters. It uses an incremental credibility-based process that is more in line with the practical application of semi-supervised learning. This method reduces the amount of labeled data required for training, making it more suitable for scenarios where the labeled data is limited. Additionally, our model can also be extended to incorporate validation data if available, making it more flexible and applicable in a wider range of scenarios.

\begin{figure}
  \centering
  \includegraphics[width=0.45\textwidth]{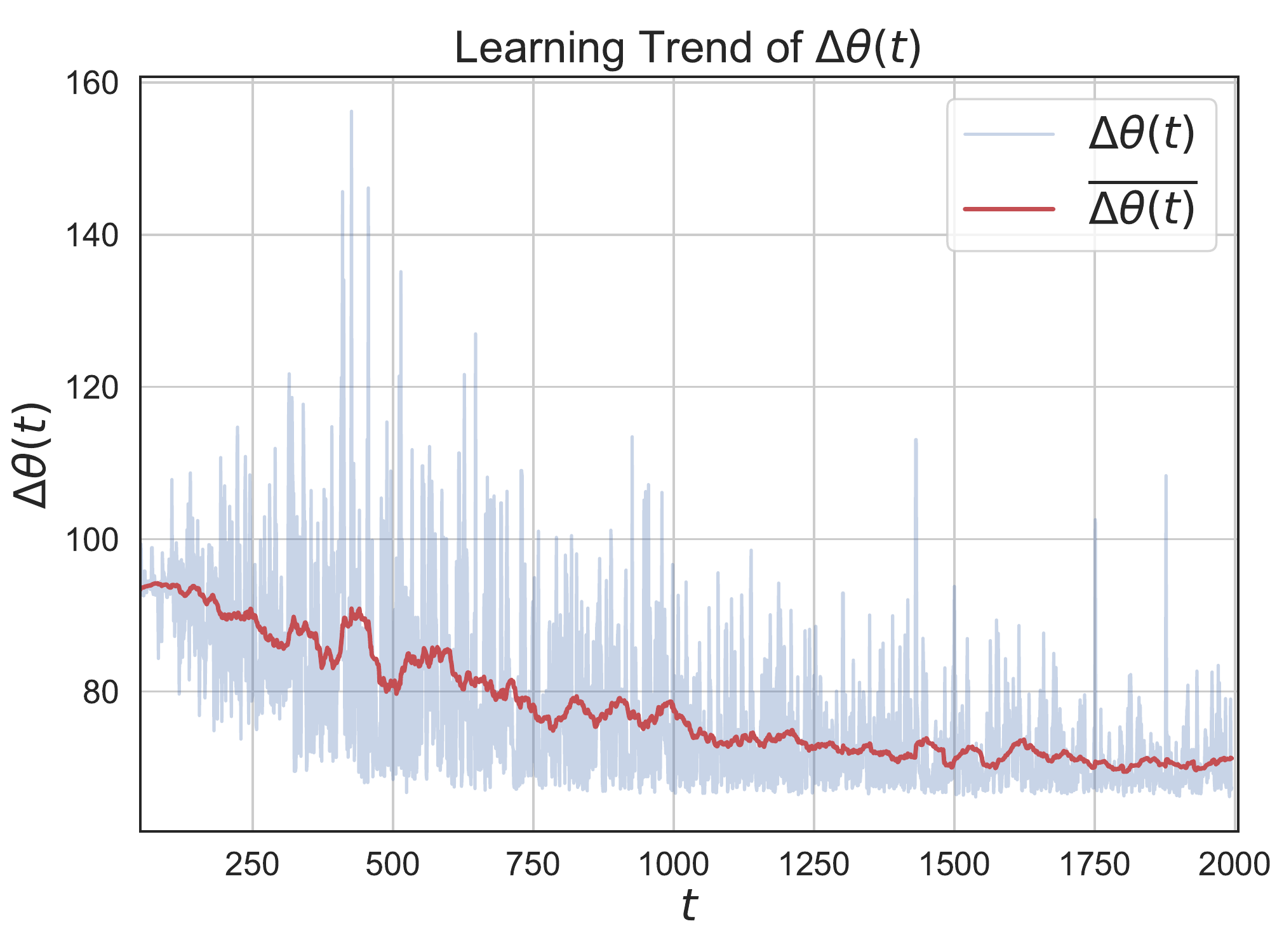}
  \caption{Trend of $\Delta \theta(t)$ in FedIL.}
  \label{Theta}
\end{figure}

\noindent\textbf{Trend of incremental $\Delta\theta(t)$.}
The proposed incremental credibility approach ensures that the learning process is gradually convergent. As illustrated in Figure~\ref{Theta}, the trend of $\Delta \theta(t)$ during the incremental learning of the proposed method can be seen, where $\overline{\Delta\theta(t)}$ represents the moving average of $\Delta\theta(t)$. Despite any perturbations that may occur during the learning process, the overall trend remains one of gradual decline, demonstrating that the proposed federated incremental strategy effectively solves the optimal mapping based on model parameter learning and leads to a convergent solution process. The experimental results further confirm the criterion outlined in Theorem~\ref{THDtheta}, indicating that the system's solution process, designed based on the criteria defined in Theorem~\ref{THDtheta}, will eventually converge to a fixed point.
%What is noticeable, the proposed incremental credibility ensures that the learning process is progressively convergent.
%Figure~\ref{Theta} shows the trend of the $\Delta \theta(t)$ during the incremental learning of the proposed method, in which $\overline{\Delta\theta(t)}$ represents the moving average of $\Delta\theta(t)$. Despite the perturbation in the process of learning, the overall trend is still gradually declining, which proves that the process of solving the optimal mapping based on model parameter learning by the system under our proposed federated incremental strategy is gradually convergent.
%The experimental results also confirm the criterion in Theorem~\ref{THDtheta}, indicating that the solution process of the system designed according to the criterion of Theorem~\ref{THDtheta} can gradually converge to a fixed point. 

\begin{figure}[h!]
  \centering
  \includegraphics[width=0.45\textwidth]{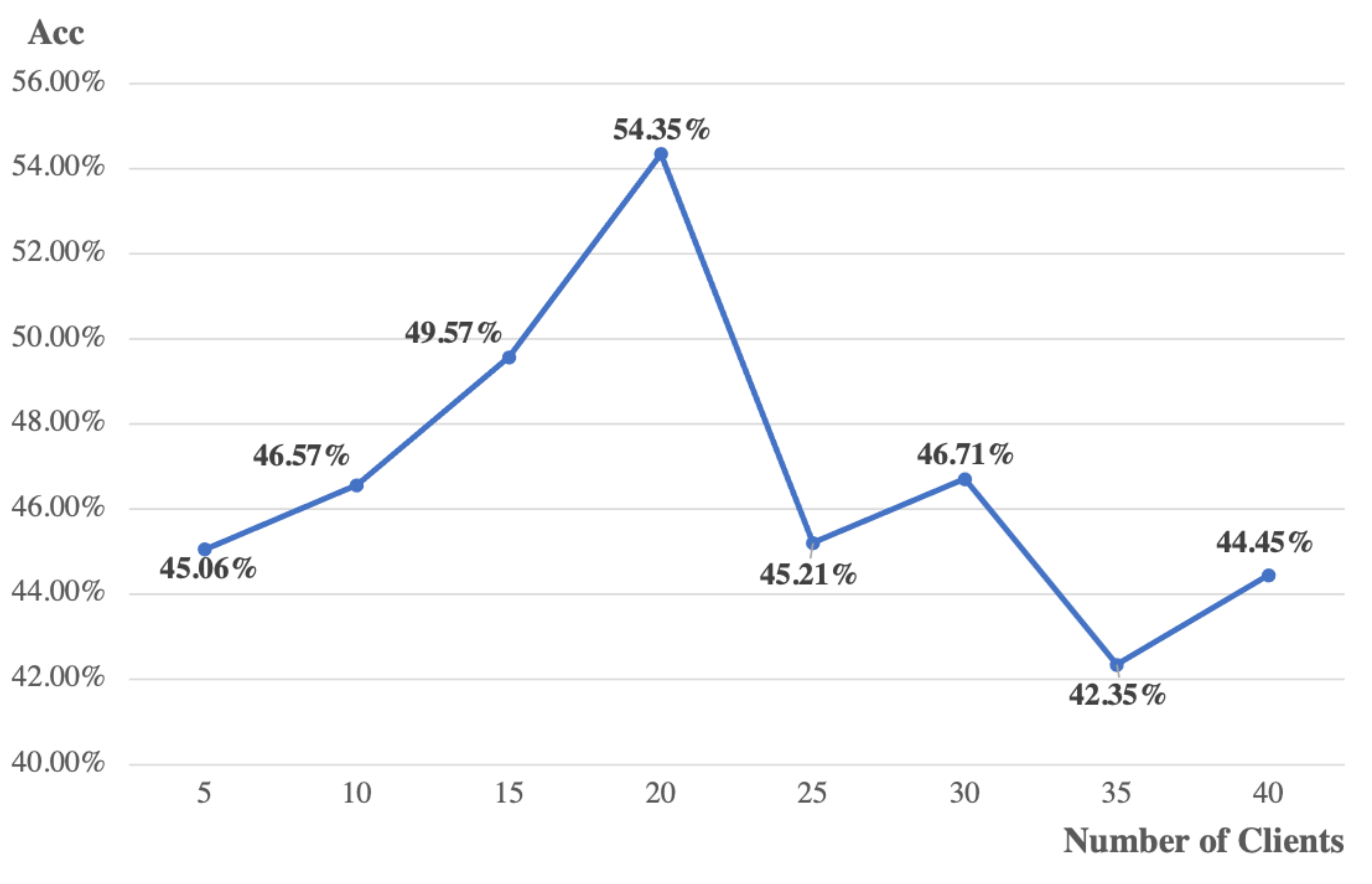}
  \caption{The Impact of Number of Clients. }
  \label{fig5}
\end{figure}

\section{Ablation Study}
\subsection{The Impact of Number of Clients.}
The impact of the number of clients participating in each training round on the performance of the global model can be seen in Figure~\ref{fig5}. Our experiments on the CIFAR10 dataset, conducted over 1000 training rounds, demonstrate a clear correlation between the number of clients selected and the global model performance. As the number of clients increases from 5 to 20, we observe a corresponding improvement in the global model's performance. However, beyond a certain point, we see that further increases in the number of clients lead to a decrease in the global model's performance. This is likely due to the fact that the data on the clients is unlabeled and incorporating too much of this unknown data in the early stages of training can slow down the convergence of the model.
%The number of clients participating in each training round will also have an impact on the performance of the global model, and the analysis results are shown in Figure~\ref{fig1}. The experimental results on CIFAR10 with 1000 training rounds clearly show a corresponding increase in the global model performance as the number of clients selected increases from 5 to 20 in each training round. However, as the number of clients continues to increase, the performance of the global model decrease. Since the data on the clients are unlabeled, learning too much unknown data in the early training stage may lead to slow convergence of the model.

\begin{figure}[h!]
  \centering
  \includegraphics[width=0.45\textwidth]{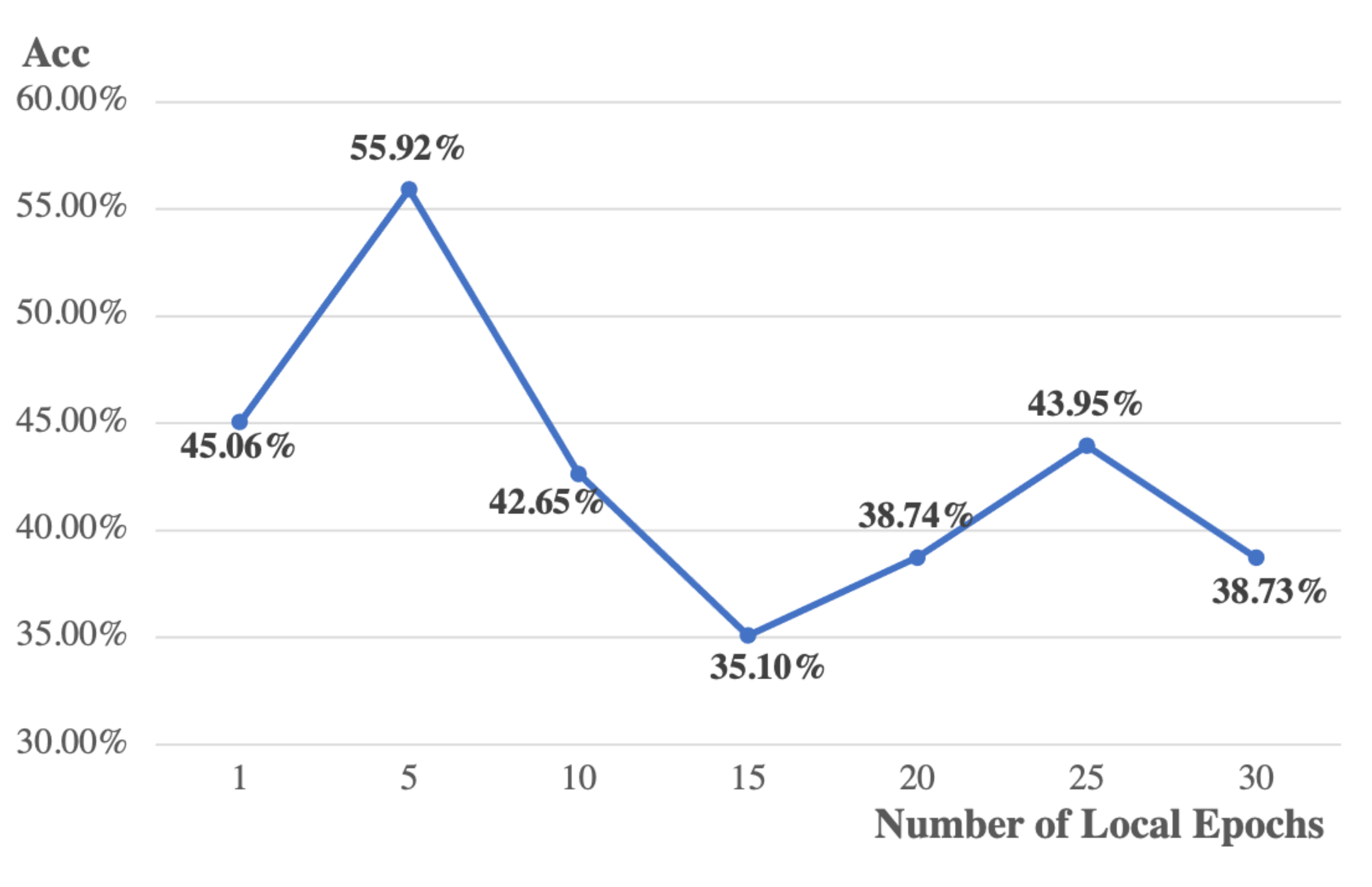}
  \caption{The Influence of Number of Local Training Epochs. }
  \label{fig6}
\end{figure}

\subsection{The Influence of Number of Local Training Epochs.}
The influence of the local training epochs in each training round on the performance of the global model is demonstrated in Figure~\ref{fig6}. The data on the clients being unlabeled, and conducting too many epochs in each client during the early stages of training can have a negative effect on the performance of the model. Our experiments, conducted on the CIFAR10 dataset over 1000 training rounds, show that the optimal number of local training epochs is 5. The results suggest that it is crucial to strike a balance between the number of epochs and the stage of training in order to maintain the optimal performance of the global model.
%The local training epochs in each training round will have an influence on the performance of the global model, and the analysis results are shown in Figure~\ref{fig2}. Since the data on the clients are unlabeled, learning too many epochs in each client in the early training stage may cause the performance of the model to degrade. The experimental result shows that the optimal number of local training epochs is 5, and the dataset used here is CIFAR10 with 1000 training rounds.

\begin{table}[h!]
\caption{Modules Analysis in FedIL.}
\centering
\small
\renewcommand\tabcolsep{5.0pt}
\begin{tabular}{lc}
\hline
Modules                                              & Acc   \\ \hline
FedIL                                                & 61.31 \\
FedIL without pseudo-label set                       & 60.77 \\
FedIL without Cosine Similarity                      & 28.59 \\
FedIL without a pseudo-label set and Cosine Similarity & 27.21 \\ \hline
\end{tabular}
\label{ablation}
\end{table}

\subsection{The Evaluation of Different Modules}
For the necessity of each module, we also did an ablation analysis. The results, as presented in Table~\ref{ablation}, demonstrate that even without the use of the pseudo-label candidate mechanism and the similar incremental screening strategy, the proposed method still achieved a close-to-30\% accuracy within a limited number of training rounds. The incorporation of the candidate label mechanism resulted in a 1\% improvement in accuracy, while the use of the similarity adjustment screening strategy resulted in an increase to over 60\%. When both the candidate tag mechanism and the similarity incremental screening were utilized simultaneously, the accuracy reached over 61\%. This highlights the significant impact of a similar incremental strategy on performance and the contribution of the candidate labeling mechanism toward accelerating the learning process. The core of this method involves the continuous enhancement of label identification by the upper-level server calculation, and the selection of candidate labels for semi-supervised learning tasks through similar incremental learning at each local client calculation in the network, which iterates over time.
%For the necessity of each module, we also did an ablation analysis. As shown in Table~\ref{ablation}, when the pseudo-label candidate mechanism and the similar incremental screening strategy are not used, the proposed method can still maintain an accuracy close to 30\% in the limited training rounds. After using the candidate label mechanism, the accuracy is improved by 1\%; After the similarity, adjustment screened the incremental updates, the progress increased to more than 60\%; when the candidate tag mechanism and the similarity incremental screening were used at the same time, the superposition of accuracy reached more than 61\%. It can be seen that a similar incremental strategy has the most significant impact on performance, and the candidate labeling mechanism accelerates learning efficiency to a certain extent. Therefore, the core of this method is to continuously increase the identification of tags by the upper-level server calculation and select candidate tags to complete the semi-supervised learning task through similar incremental learning of each local calculation in the network with the iteration.

\begin{table}[h!]
\caption{Accuracy Stats of FedIL on Various Count of Successive Labeling $k$. The dataset used here is CIFAR10.}
\centering
\small
\renewcommand\tabcolsep{1.0 pt}
\begin{tabular}{lccccccc}
\hline
\multicolumn{1}{c}{\begin{tabular}[c]{@{}c@{}}Count of Successive \\ Labeling $k$\end{tabular}} & 2                                                     & 3                                                     & 4                                                     & 5                                                     & 6                                                     & 7                                                              & 8                                                     \\ \hline
Accuracy                                                                                      & \begin{tabular}[c]{@{}c@{}}52.95\\ ±0.13\end{tabular} & \begin{tabular}[c]{@{}c@{}}54.03\\ ±0.35\end{tabular} & \begin{tabular}[c]{@{}c@{}}54.10\\ ±0.26\end{tabular} & \begin{tabular}[c]{@{}c@{}}59.51\\ ±0.17\end{tabular} & \begin{tabular}[c]{@{}c@{}}60.27\\ ±0.26\end{tabular} & \begin{tabular}[c]{@{}c@{}}61.31\\ ±0.25\end{tabular} & \begin{tabular}[c]{@{}c@{}}62.03\\ ±0.37\end{tabular} \\ \hline
\end{tabular}
\label{tab1}
\end{table}

\subsection{Count of Successive Labeling $k$ and Threshold $\tau$}
According to the trend shown in Table~\ref{tab1}, it is evident that there is a trade-off between the reliability of the candidate tag data and calculation time. A larger value of $k$ results in more reliable candidate labels, but it also increases the calculation time. Hence, when selecting parameters, it is essential to consider not only the performance metrics but also the distribution characteristics, computational costs, and time constraints. In our study, we have chosen $k=7$ and threshold $\tau$ equal to 0.95, which is consistent with the parameters used in the Fixmatch \cite{DBLP:conf/nips/SohnBCZZRCKL20}. By balancing the competing demands of accuracy and efficiency, we aim to achieve optimal results in our experiments.

%According to the trend shown in Table~\ref{tab1}, within a certain range, the larger $k$ is, the more reliable the selected candidate tag data is, but the longer the calculation time is required. Therefore, the selection of parameters should not only focus on a certain performance but also consider the characteristics of distribution and factors such as computational and time costs. In this paper, we use $k=7$. For threshold $\tau$, we use 0.95, which is the same as the Fixmatch \cite{DBLP:conf/nips/SohnBCZZRCKL20}.

%We also analyze the impact of the number of clients participating in each training round, and the influence of the number of training epochs for each client in one training round, which are illustrated in the supplementary materials.

\section{Conclusion}
This paper proposes Federated Incremental Learning (FedIL) with an incremental credibility learning strategy to select highly credible pseudo-labels to join the pseudo-label set to establish complete semi-supervised training on unknown samples in each client. We utilize cosine similarity to select client weights and propose a standardized global incremental learning framework to speed up training models while ensuring server-client consistency between learning supervised and unsupervised tasks. Based on Banach's fixed point theorem, we further prove the necessary and sufficient conditions for the convergence of global incremental learning, revealing that the system enters a progressive convergence stage when the weight increment difference between server and client decreases.

Experimental validation shows that the proposed FedIL outperforms most of the FSSL baselines on common public datasets, thus demonstrating the feasibility of the proposed FedIL mechanism for performance improvement in federated semi-supervised learning by determining the estimation of unlabeled data with incremental learning. Due to the existing hardware limitations of client devices, federated learning can only be implemented with small backbones and datasets, e.g., MNIST, CIFAR10 and CIFAR100. In future work, we will investigate the applicability of large-scale datasets, e.g., ImageNet, in the FSSL scenario.

\bibliographystyle{ieeetr}
\bibliography{ref}
\end{document}